\documentclass[accepted]{uai2026} 

\usepackage[dvipsnames]{xcolor} 
\usepackage{booktabs}
\usepackage{multirow}
\usepackage[american]{babel}

\usepackage{natbib} 
\usepackage{mathtools} 
\usepackage{stmaryrd}
\usepackage{amsmath,amssymb,amsthm}

\usepackage{graphicx}
\usepackage{tikz}
\usepackage{float}
\usepackage{subcaption}
\usepackage{url}
\usepackage{epstopdf}
\usepackage[normalem]{ulem}

\usepackage{epstopdf}

\usepackage{algorithm}
\usepackage{algorithmic}

\usepackage{url}
\usepackage{graphicx}
\usepackage{subcaption}

\newtheorem{theorem}{Theorem}[section]
\newtheorem{lemma}[theorem]{Lemma}
\newtheorem{proposition}[theorem]{Proposition}
\newtheorem*{remark}{Remark}

\newcommand\smallmath[2]{#1{\raisebox{\dimexpr \fontdimen 22 \textfont 2
      - \fontdimen 22 \scriptscriptfont 2 \relax}{$\scriptscriptstyle #2$}}}
\newcommand\smallplus{\smallmath\mathbin +}
\newcommand\smallminus{\smallmath\mathbin -}

\def\E{{\mathbb{E} }}

\def\bk{{\mathbf{k}}}
\def\bl{{\mathbf{l}}}

\def\pcm{{\Sigma(v)}} % real posterior cov matrix
\def\pcmh{{\hat\Sigma(v)}} % estimated posterior cov matrix
\def\ncmv{{R_\tau(v)}} % observation noise cov matrix given v
\def\ncmvh{{\hat R_\tau(v)}} % estimated observation noise cov matrix given v
\def\aca{{M}} % (A*PCM*A^T)^{-1}
\def\acah{{\hat M}} % (A*PCMH*A^T)^{-1}
\def\pvvh{{\hat \nu(v)}} % posterior variance vector hat
\title{Self-Supervised Uncertainty Estimation For Super-Resolution of Satellite Images}

\author[1]{\href{mailto:<zhe.zheng@ens-paris-saclay.fr>}{Zhe~Zheng}{} }
\author[1]{Valéry~Dewil}
\author[2]{Pablo~Arias}
\affil[1]{
    Centre Borelli \\
    ENS Paris-Saclay \\
    Universite Paris-Saclay \\
    CNRS, Gif-sur-Yvette, France
}
\affil[2]{%
    Dept.\ of Engineering \\
    Universitat Pompeu Fabra \\
    Barcelona, SPAIN
}

\begin{document}

\maketitle

\begin{abstract}
Super-resolution (SR) of satellite imagery is challenging due to the lack of paired low-/high-resolution data. Recent self-supervised SR methods overcome this limitation by exploiting the temporal redundancy in burst observations, but they lack a mechanism to quantify uncertainty in the reconstruction. In this work, we introduce a novel self-supervised loss that allows 
to estimate uncertainty in image super-resolution without ever accessing the ground-truth high-resolution data. 
We adopt a decision-theoretic perspective and show that minimizing the corresponding Bayesian risk yields the posterior mean and variance as optimal estimators. We validate our approach on a synthetic SkySat L1B dataset and demonstrate that it produces calibrated uncertainty estimates comparable to supervised methods. 
Our work bridges self-supervised restoration with uncertainty quantification, making a practical framework for uncertainty-aware image reconstruction.
\end{abstract}

\begin{figure*}[ht!]
    \begin{minipage}{0.09\linewidth}
    \begin{subfigure}[b]{\linewidth}
    \includegraphics[trim={190 46 27 173}, clip,width=\linewidth]{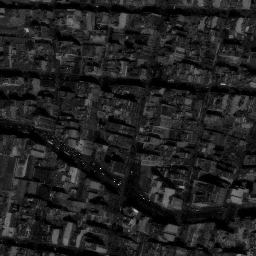}
    \includegraphics[trim={190 46 27 173}, clip,width=\linewidth]{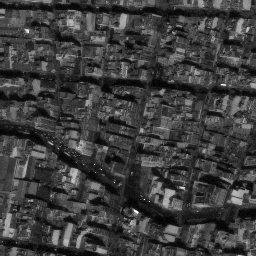}
    \includegraphics[trim={190 46 27 173}, clip,width=\linewidth]{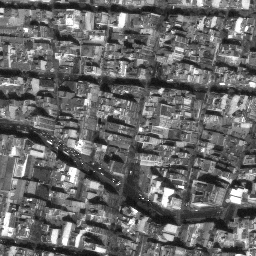}
    \includegraphics[trim={190 46 27 173}, clip,width=\linewidth]{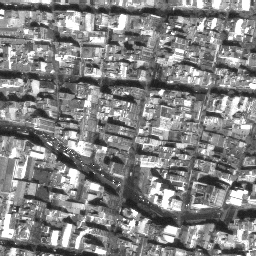}
    \end{subfigure}
    \end{minipage}%
    \begin{minipage}{0.55\linewidth}
    \begin{center}
        \begin{tikzpicture}%
        \node[anchor=south west, inner sep=0] (X) at (0,0){\includegraphics[trim={383 92 54 347}, clip,width=0.323\textwidth]{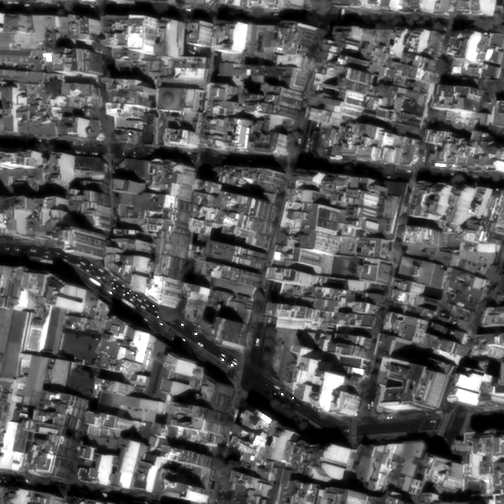}};
        \node[above right=.1em, yshift=0.25\textwidth, fill=white] (Z) {\small Self-sup. 47.56 dB};
        \end{tikzpicture}
        \begin{tikzpicture}%
        \node[anchor=south west, inner sep=0] (X) at (0,0){\includegraphics[trim={383 92 54 347}, clip,width=0.323\textwidth]{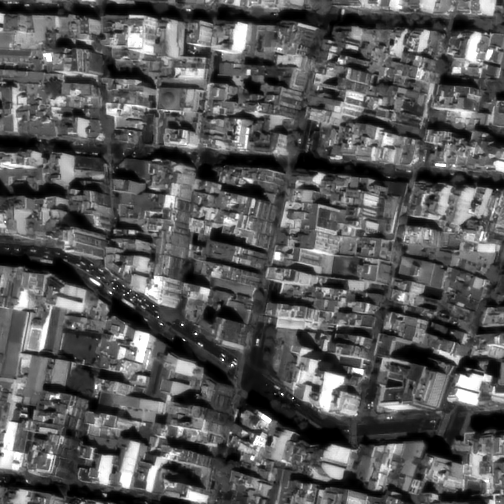}};
        \node[above right=.1em, yshift=0.25\textwidth, fill=white] (Z) {\small Sup. 47.74 dB};
        \end{tikzpicture}
        \begin{tikzpicture}%
        \node[anchor=south west, inner sep=0] (X) at (0,0){\includegraphics[trim={383 92 54 347}, clip,width=0.323\textwidth]{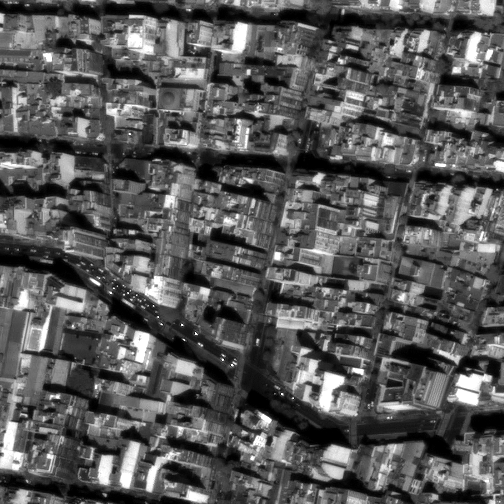}};
        \node[above right=.1em, yshift=0.25\textwidth, fill=white] (Z) {\small GT};
        \end{tikzpicture}

        \begin{tikzpicture}%
        \node[anchor=south west, inner sep=0] (X) at (0,0){
        \includegraphics[trim={383 92 54 347}, clip,width=0.323\textwidth]{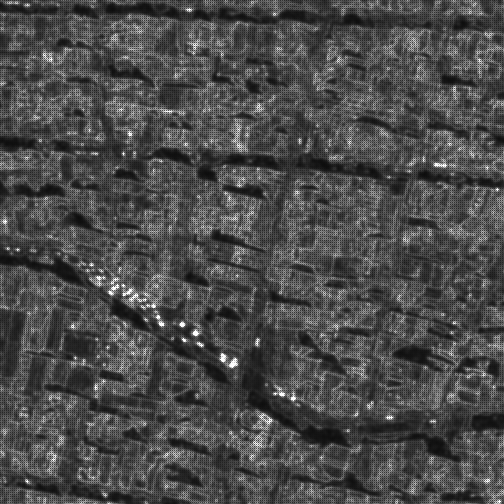}};
        \node[above right=.1em, yshift=0.25\textwidth, fill=white] (Z) {\small Self-sup.};
        \end{tikzpicture}
        \begin{tikzpicture}%
        \node[anchor=south west, inner sep=0] (X) at (0,0){
        \includegraphics[trim={383 92 54 347}, clip,width=0.323\textwidth]{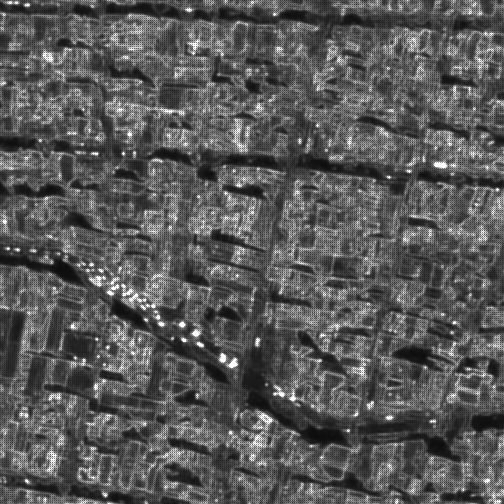}};
        \node[above right=.1em, yshift=0.25\textwidth, fill=white] (Z) {\small Sup.};
        \end{tikzpicture}
        \begin{tikzpicture}%
        \node[anchor=south west, inner sep=0] (X) at (0,0){
        \includegraphics[trim={383 92 54 347}, clip,width=0.323\textwidth]{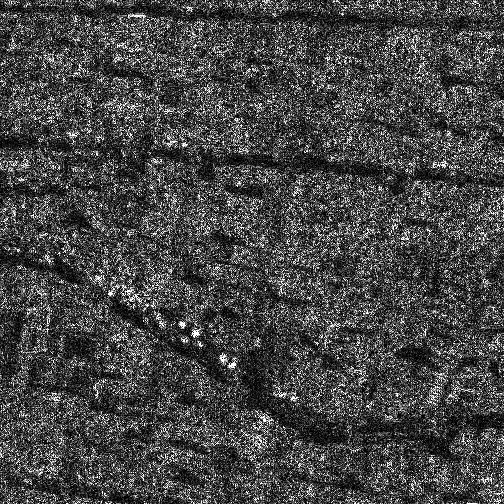}};
        \node[above right=.1em, yshift=0.25\textwidth, fill=white] (Z) {\small Recons. error};
        \end{tikzpicture}

    \end{center}
    \end{minipage}%
    \begin{minipage}{0.36\textwidth}
    \centering
    \includegraphics[width=\linewidth]{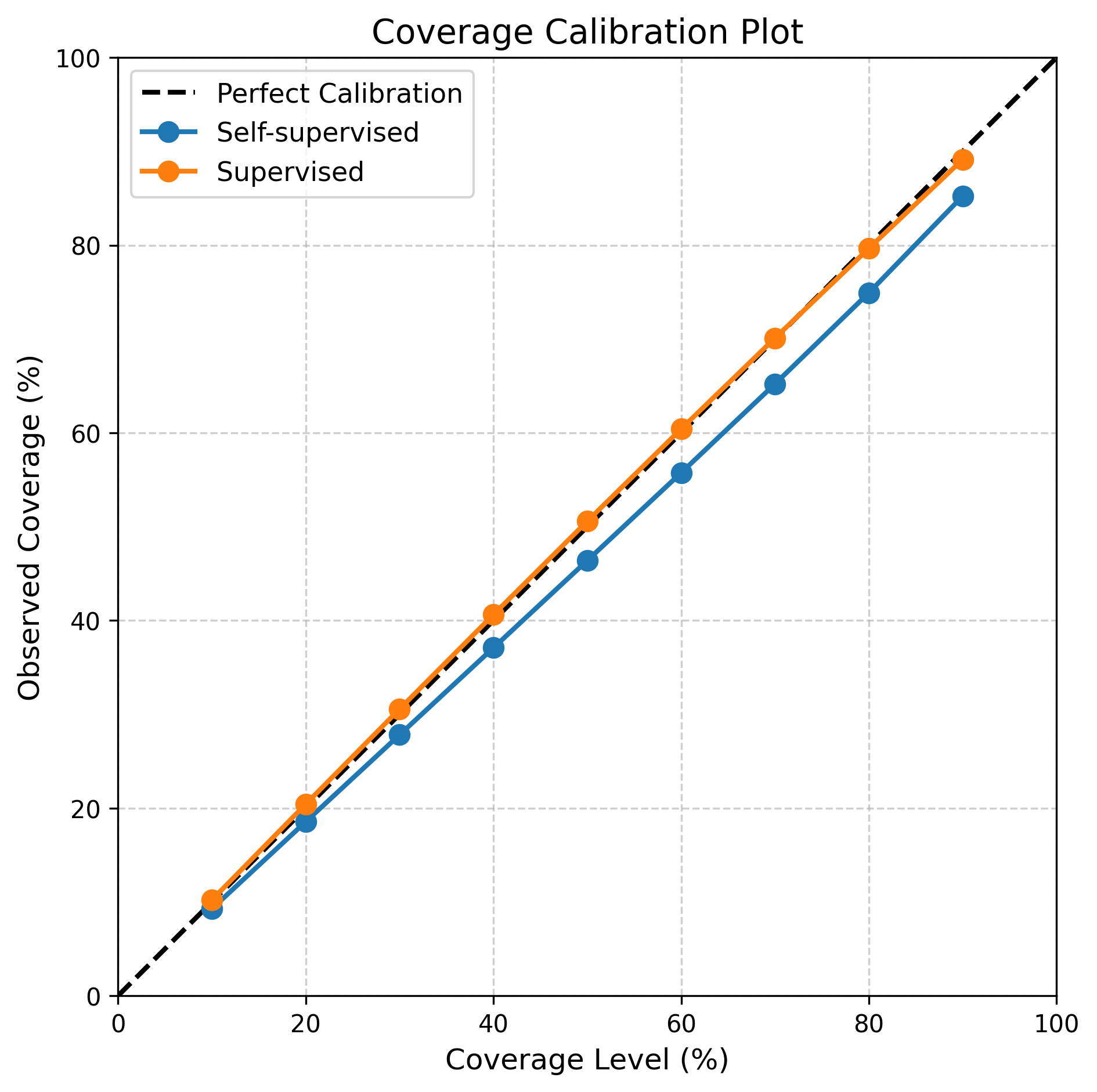}
    \end{minipage}
    
    \caption{
    We propose a self-supervised strategy to train a network that performs super-resolution with uncertainty quantification in satellite imaging. 
    The network is purely trained using noisy and low-resolution data, without requiring high-resolution ground truth. \textbf{Left}: four input images out of a burst of multi-exposure, low-resolution and noisy images; \textbf{Center}: Reconstruction results (top) and estimated uncertainty (bottom) by the self-supervised network and by the supervised one; high resolution ground truth (top) and reconstruction error by the self-supervised method. 
    \textbf{Right}: comparison of the coverage test for the estimated uncertainty (variance of posterior distribution) with the self-supervised and supervised training.}
    \label{fig:comparison_reconstruction}
\end{figure*}
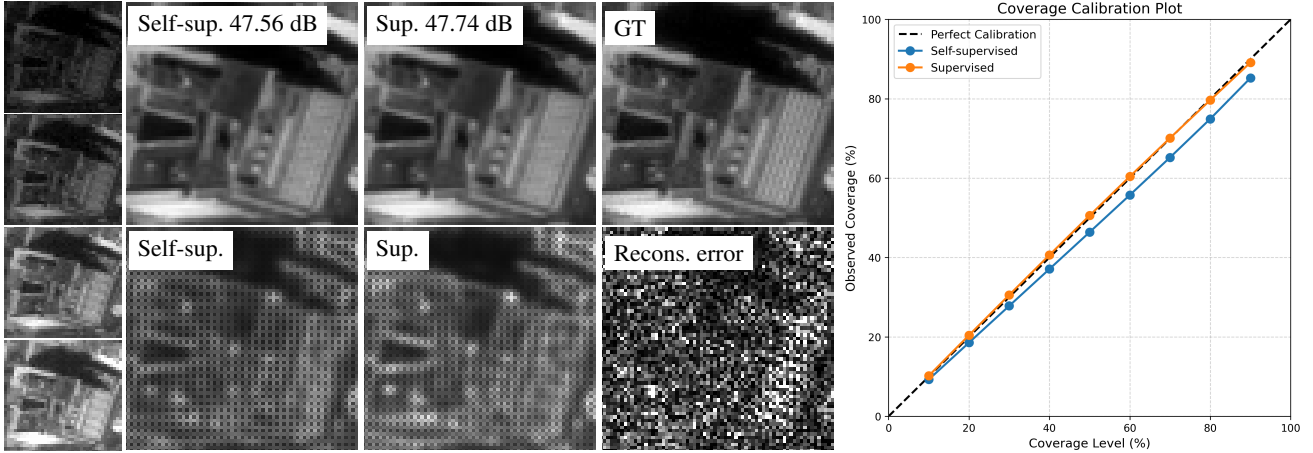

\section{Introduction}\label{seq:intro}
 
\sloppy

Satellite image super-resolution is crucial for applications like human activity monitoring and disaster relief. 
Image super-resolution has achieved remarkable progress with deep neural networks. 
However, most approaches rely on supervised training, requiring paired low and high-resolution data. This is impractical for some real use cases such as satellite image processing where ground-truth high-resolution (HR) images are unavailable. Recent self-supervised methods~\citep{nguyen2021self,nguyen2022self,bhat2023self} have shown that that high-quality reconstructions can be obtained purely from redundant observations, without the need of ground-truth HR images.
These approaches produce deterministic point estimates and lack a way to quantify the uncertainty in the reconstruction. As a result, they can yield visually plausible but unreliable predictions, which could lead to an incorrect interpretation of the images.

Uncertainty estimation has been widely explored in deep learning and applied in vision tasks~\citep{gawlikowski2023survey}.
A variety of approaches have been proposed for estimating the uncertainty~\citep{gawlikowski2023survey,angelopoulos2022image,kendall2017uncertainties,valsesia2021permutation}. Uncertainty can be caused by multiple sources~\citep{hullermeier2021aleatoric} and is usually categorized into two types: aleatoric and epistemic~\citep{hullermeier2021aleatoric,kendall2017uncertainties}. 
Epistemic uncertainty represents our lack of knowledge over the optimal model for a given inference problem, and it could be reduced by better training and/or modeling. Aleatoric uncertainty is associated to the inherent randomness of the inference problem itself and cannot be reduced.

For point estimation problems,
aleatoric uncertainty is typically quantified by measurements of the spread of the (true) posterior distribution, such as variance~\citep{kendall2017uncertainties}, mean absolute deviation~\citep{valsesia2021permutation}, or predictive quantiles~\citep{angelopoulos2022image}. This can be achieved by training networks with suitable loss functions.
A common approach is to estimate the posterior mean and variance
via the minimization of a loss function based on the Gaussian negative log-likelihood \citep{nix1994estimating,bishop1994mixture}.
Such loss functions however, require supervision from the ground truth
which limits their application for image restoration problems in which ground truth is hard or impossible to access.

In this work, we aim to incorporate uncertainty estimation into self-supervised super-resolution. Building upon the recent self-supervised framework of~\citet{nguyen2022self}, we extend it to explicitly model pixel-wise uncertainty.
Our contributions can be summarized as follows:

\emph{(1)}
We introduce a novel self-supervised loss for super-resolution based on the Gaussian negative log-likelihood (NLL), to jointly estimate pixel-wise posterior means and variances, enabling uncertainty-aware reconstruction (Section \ref{sec:method}).
Rather than assuming that the posterior distribution 
follows a Gaussian distribution, we take a decision-theoretic perspective 
and show that minimizing the corresponding risk yields both the mean and variance of the posterior distribution as optimal estimators. 

\emph{(2)}
We derive the optimal estimators under the resulting Bayesian risk, providing a theoretical justification for the approach (Section \ref{sec:optimal_estimator}).
In particular, we show that for a subsampling degradation operator and under Gaussian degradation noise (possibly signal dependent), the proposed loss is minimized by the posterior mean and pixel-wise variances, thus being equivalent to a supervised training with the Gaussian NLL loss. 
    
\emph{(3)} We validate our method on a synthetic SkySat's L1B dataset generated following~\citet{nguyen2022self}, showing calibrated uncertainty estimates comparable to those obtained with supervised Gaussian NLL (Section \ref{sec:results}).
    
By bridging the self-supervised restoration with uncertainty quantification, our method makes a practical framework for uncertainty-aware satellite image reconstruction.
While we focus on the super-resolution problem, our analysis assumes a more general linear degradation model which could be useful for uncertainty quantification in other 
inverse problems in imaging where ground truth data is not available.

\section{Related Work}
\label{sec:related_work}

\textbf{Uncertainty estimation.}
Uncertainty is commonly categorized into aleatoric and epistemic uncertainty~\citep{hullermeier2021aleatoric}, although there is no widespread agreement on a precise definition of these categories \citep{valdenegro2022deeper,gruber2025sources,kirchhof2025reexamining,bickford2025rethinking}.
Bayesian neural networks~\citep{blundell2015weight,gal2016dropout,gal2015bayesian} formalize epistemic uncertainty as a probability over the model parameters conditioned on the training dataset while the aleatoric uncertainty is formalized as the true posterior probability of the label conditioned on an input data during inference. Deep ensembles \citep{lakshminarayanan2017simple} and Monte Carlo Dropout~\citep{gal2016dropout} have also been 
proposed to approximate the distribution on the weights. 
Aleatoric uncertainty is typically captured by maximizing the likelihood of the ground truth data in a supervised setting~\citep{nix1994estimating,bishop1994mixture,kendall2017uncertainties}. This is often motivated with a probabilistic regression perspective in which the outputs are assumed to have Gaussian noise.

In this work, we extend the Gaussian NLL loss to a self-supervised setting, where instead of a HR clean ground truth label we use a LR noisy one. 
We motivate our loss via 
a decision-theoretic viewpoint: we show that the minimization of the risk induced by the loss yields the posterior mean and variance of the high resolution target. This perspective emphasizes expected prediction error (as in \citep{bickford2025rethinking}) rather than strict probabilistic calibration (as in \citep{lahlou2023deup}), and highlights the fact that the trained network will approximate the posterior variance \emph{even if the true posterior distribution is not Gaussian}.

\textbf{Uncertainty in imaging problems.}
A line of uncertainty estimation in imaging problems is leveraging the priors encoded by networks and sampling from the approximate posterior distribution, such as \citep{chung2023diffusion,altekruger2023wppnets,kawar2022denoising}.
Conformal prediction is applied to imaging restoration by~\citet{angelopoulos2022image}. This method constructs valid confidence intervals by calibrating a heuristic estimate on a calibration dataset. It is a post-processing step without modifying the underlying model; thus, conformal correction is directly applicable to the method proposed in this paper. 
\citet{tachella2024equivariant} recently proposed Equivariant Bootstrapping. This approach constructs i.i.d. samples that approximate the estimator's distribution, allowing for quantification of risk between the reconstruction and ground truth.

Uncertainty estimation has received significant attention in satellite imaging. 
\citet{ebel2023uncrtaints} predicts a pixel-wise aleatoric uncertainty map to quantify the reliability of the reconstruction  for multi-temporal cloud removal. 
\citet{valsesia2021permutation}  estimates per-pixel aleatoric uncertainty for multi-temporal super-resolution, by using a Laplacian NLL loss in a supervised manner.
Uncertainty quantification is also well-established in other downstream remote sensing applications beyond image reconstruction, including road extraction in SAR imagery~\citep{haas2021uncertainty}, 
atmospheric states estimation~\citep{braverman2021post}, crop yields, PM2.5 estimation~\citep{miranda2025analysis}.

These approaches require ground-truth images to be trained.

\textbf{Self-supervised learning for image reconstruction.} Self-supervised learning has become an effective strategy for image restoration when clean labels are unavailable~\citep{tachella2026self}. 
The pioneering Noise2Noise framework~\citep{lehtinen2018noise2noise} showed that a denoising network can be trained with the MSE loss using pairs of independently corrupted images, leading to the same optimal estimator of the clean image.
Subsequent methods removed the need for paired datasets: blind-spot methods~\citep{krull2019noise2void,batson2019noise2self},
SURE-based approaches
\citep{metzler2018unsupervised,soltanayev2018training,tachella2024unsure} and methods based on adding more noise to the already noisy images 
\citep{kim2021noise2score,pang2021recorrupted,moran2020noisier2noise}.

These approaches has been generalized to video~\citep{ehret2019model,dewil2021self}, and to inverse problems with a linear degradation operators, such as a demosaicking~\citep{ehret2019joint}, burst super-resolution~\citep{nguyen2021self,nguyen2022self,bhat2023self},
MRI reconstruction 
\cite{aggarwal2022ensure,millard2023theoretical}, CT reconstruction \cite{hendriksen2020noise2inverse,chen2021equivariant}, video microscopy denoising \cite{aiyetigbo2024unsupervised}; and for training diffusion models from incomplete observations \cite{daras2023ambient}.

\citet{altekruger2023wppnets}, \citet{laine2019high}, and \citet{krull2020probabilistic} proposed self-supervised image restoration methods with uncertainty quantification. 
The first paper proposed training a conditional normalizing flow in an unsupervised way for super-resolution, allowing sampling from the posterior distribution.
In the latter papers, the denoising network outputs a parameterization of the posterior probability, and is trained by maximizing the likelihood of the noisy image. \citet{krull2020probabilistic} considers discrete output images with integer values in $[\![0,255]\!].$ The network outputs the probability mass function for each pixel. In \citet{laine2019high} the posterior probability is assumed Gaussian, which motivates the use of a Gaussian NLL loss between the network output and the noisy target. Our loss can be seen as a generalization of this approach to the case in which the observation is not only noisy, but also subsampled.

\section{Method}
\label{sec:method}

\subsection{Self-Supervised MISR}\label{sec:misr}

The proposed method (shown in Figure \ref{fig:data_flow}) builds upon the recent self-supervised MISR method~\citep{nguyen2022self}, which trains multi-image super-resolution (MISR) networks without requiring corresponding high-resolution (HR) ground-truth images. 
The self-supervised loss for MISR was first introduced in \citet{nguyen2021self}, and then was extended \citep{nguyen2022self} to handle low-resolution bursts $\{v_t\}_{t=1}^N$, where $v_i \in \mathbb{R}^q$,  captured with varying exposure times.
We summarize it below.

\paragraph{Degradation model for satellite MISR.} 
Each low-resolution (LR) satellite frame $v_t$ is of size $H\times W$ and is modeled as a blurred, warped, sampled, and noisy observation of an underlying infinite-resolution scene $\mathcal{I}$. Specifically, for frame $t$,
$v_{t} \;=\; \Pi \, F_t \bigl( \mathcal{I} * k \bigr) \;+\; n_t,$
where $k$ is the point spread function that jointly accounts for optical blur and pixel integration, $F_t$ is the frame-dependent geometric transform (such as a homography, or an affine transform),
$\Pi$ is the 2D sampling operator of the sensor array,
and $n_t$ represents additive noise. 
In \emph{push-frame} constellations such as SkySat, the optical cutoff of $\mathcal I \ast k$ is about twice the LR sampling rate, implying that practically achievable super-resolution is limited to a $2\times$. 
By denoting by $u$ a HR sampling of size $2H\times 2W$ of the ideal continuous image $\mathcal I\ast k$, we have the following degradation model: 
\begin{equation}
  v_{t} \;=\; D \, F_t u \;+\; n_t,
  \label{eq:degradation}
\end{equation}
where $D \in \mathbb{R}^{HW \times 4HW}$ is a subsampling operator by a factor of two, which simply samples the pixels with even coordinates. The goal is to 
estimate $u$.
The first LR image \(v_1\) is considered as the \emph{reference}, and it is assumed that it is aligned with the high resolution image (i.e. $F_1 = I$).

\paragraph{Network architecture.} 

\begin{figure}
    \centering
    \includegraphics[width=1.0\columnwidth]{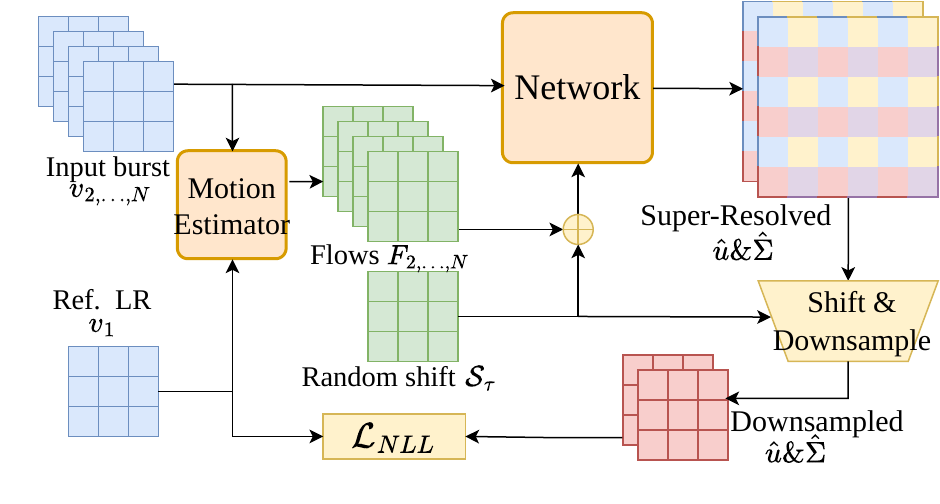}
    \caption{Data flow during training. The optical flows $\{F_t\}_{i=2}^N$ are substracted first by the shift $\mathcal{S}_{\tau}$ and then sent to the network together with the burst LR observations $\{v_t\}_{i=2}^N$. The super-resolved image $\hat u(v)$ and $\hat\Sigma$ are shifted by $\mathcal{S}_{\tau}$ and then downsampled to compute loss with the reference $v_1$.}
    \label{fig:data_flow}
\end{figure}

The HDR-DSP (for \emph{High Dynamic Range Deep Shift-and-Pool})  network was designed by~\citet{nguyen2022self} for handling multi-exposure LR inputs, with inaccurate exposure times. It takes as input the burst of LR images $\{v_t\}_{t=1}^N$ (between 5 and 20 images), together with their corresponding exposure times. The LR images are first normalized by dividing them by the exposure time, and then decomposed into base and detail components with Gaussian low-pass and high-pass filters. The network processes the detail components, while the base components
which are easy to super-resolve, are averaged to reduce noise and upscaled using bilinear interpolation.

The network consists of three trainable modules: encoder, fusion and decoder.
The encoder extract features for each LR input frame. These features 
processed by the fusion network which fuses them into a single HR feature map,
which is finally processed by the decoder to compute the HR output image.
The fusion network takes as input an estimate of the motion $F_t$, which is required for a precise placement of the encoded LR features in the high resolution grid. The motions $F_t$ are estimated by a motion estimation network which computes the optical flow between each input LR frame $v_t$ and the reference $v_1$. 

While the specific details of the architecture are not relevant, the following properties are important for the self-supervised training. \textit{(i)} the network produces an HR image aligned to the reference frame; \textit{(ii)} the network can process any number of input images (and it is invariant to their order). The last property is inherited from the pooling operators used for feature fusion.

\paragraph{Self-supervised loss.} 
The self-supervised MISR method utilizes a burst of  degraded LR observations, $\{v_t\}_{t=1}^N$ 
to produce a HR reconstruction without requiring access to the corresponding ground-truth HR image, $u$. 
Instead, it leverages the temporal redundancy within the burst, following a strategy inspired by N2N~\citep{lehtinen2018noise2noise} and F2F~\citep{ehret2019model,dewil2021self}.
During the training phase, 
the reference frame $v_1$ is excluded from the network's input to serve as the target in the loss.
For the self-supervised training, \citet{nguyen2022self} artificially modified 
the degradation model by
introducing a random shift operator $S_\tau$
\begin{equation}
v_t = DF_tS_\tau u + n_t.
\label{eq:degradation-model-grid-shifting-z}
\end{equation}
$S_\tau$ shifts the image by $\tau$ pixels,
with $\tau \sim \mathcal U\{0,1\}^2$ (its justification will be given shortly).
This amounts to changing the assumed position of the high resolution image.
For the reference frame, $F_1 = I$ and we have that $
v_1 = DS_\tau u + n_1.$
The self-supervised loss applies this degradation operator to $\hat u = \hat u(\{v_t\}_{t=2}^N)$
\[\ell_{ss}(\hat u, v_1) = \|DS_\tau \hat{u} - v_1\|^2.\]
Note that without $S_\tau$, only the output pixels sampled by $D$ (those with even coordinates) would be ``seen'' by the loss. 
Instead $S_\tau$ shifts the super-resolved image before down-sampling, and thus each of the four possible subsamplings of $\hat u(v)$ has a probability 1/4 of appearing in the loss.
Note that the shift $\tau$ has to be added to the estimates of the motion $F_t$ computed by the motion estimation network. This keeps the alignment between the between $DS_\tau \hat u$ and $v_1$. 

In what follows, to simplify notation, we denote the reference as $z = v_1$ and the network's input $v = \{v_t\}_{t=2}^N$.

\subsection{Self-Supervised Gaussian NLL Loss}

As discussed above, the Gaussian negative log-likelihood (NLL) is commonly used as a loss function for uncertainty estimation. 
Most frequently, it is used to compute the pre-pixel posterior variances: 
\begin{equation}
\mathcal{L}^{\text{sup}}_{\text{NLL}}(\hat{u}(v), \hat \sigma(v), u) = \sum_i\frac{(u_i - \hat{u}(v)_i)^2}{2\hat \sigma^2(v)_i}  + \frac{1}{2}\log \hat \sigma^2(v)_i.
\label{eq:sup-gaussian-nll}
\end{equation}
From the perspective of minimizing the empirical risk function, this formulation allows the model to predict both accurate reconstructions, the posterior mean $\hat{u}(v) = \mathbb{E}_{u\mid v}[u]$, and uncertainty estimates $\hat \sigma^2(v) =\mathbb{E}_{u\mid v}[(u - \hat{u}(v))^2] $. 
We emphasize here that this remains true no matter the posterior distribution - i.e. the fact that the loss corresponds to a NLL of a Gaussian distribution does not limit its validity to Gaussian posteriors. 
This fact seems to be often overlooked in the literature about uncertainty quantification, where the Gaussian NLL is usually motivated by assuming a Gaussian probabilistic model for the posterior \citep{nix1994estimating,bishop1994mixture,kendall2017uncertainties,lakshminarayanan2017simple}. 

Inspired by this, to incorporate the uncertainty estimation into the self-supervised framework, we penalize the Gaussian NLL of the degraded target $z$.

We consider a general linear degradation operator \(A_\tau\) which depends on a random variable $\tau$, i.e. 
$z = A_\tau u + n,$
where $n$ is signal-dependent additive Gaussian white noise.
While in our case, $A_\tau$ corresponds to the random shifts $S_\tau$, other self-supervised methods consider different random linear degradations \citep{ehret2019model,dewil2021self,aggarwal2022ensure,millard2023theoretical}. The proposed loss and Lemmas \ref{lemma:optimal_condition_full} and \ref{lemma:optimal_condition_diagonal} can be applied in those cases.

We will consider that the distribution of $\tau$ might depend on $v$.
If $u|v$ has mean $\hat u(v)$ and covariance matrix $\pcmh$, it follows from the degradation model that $z|v,\tau$ has mean $A_\tau \hat u(v)$ and covariance matrix $A_\tau \pcmh A_\tau^T + \ncmv$, where $\ncmv = \mathbb E_{n,u|\tau,v}\{nn^T\}$. 
This motivates the following self-supervised NLL loss:
\begin{align}
\mathcal{L}&_{\text{NLL}}(\hat{u}(v),  \pcmh, z) \notag \\
=& \frac{1}{2} (z \smallminus A_\tau \hat{u}(v))^T (A_\tau \pcmh A_\tau^T \smallplus \ncmvh)^{-1} (z \smallminus A_\tau \hat{u}(v))  \notag \\
+& \frac{1}{2} \log \det (A_\tau \pcmh A_\tau^T \smallplus \ncmvh).\label{eq:nll-self}
\end{align}
Here, $\hat u(v)$ and $\pcmh$ are the estimators outputted by the neural network receiving as input the LR sequence $v = \{v_i\}_{i=2}^N$. 
During training, 
$A_\tau$ is assumed to be known, but $\ncmv$ is not and needs to be estimated from the available data (e.g. from the degraded target $z$, as will be explained shortly).
We denote by $\ncmvh$  the estimated observation noise covariance. 
In what follows we will study the optimal estimators of the corresponding Bayes risk:
\begin{equation}
\mathcal{R}\bigl[\hat u, \hat \Sigma \bigr]
=
\mathbb{E}_{v}\mathbb{E}_{\tau, z \mid v}\Bigl[\mathcal{L}_{\mathrm{NLL}}(\hat{u}(v), \hat \Sigma(v), z)
\Bigr].\label{eq:self-sup-risk-nll}
\end{equation}
During training we minimize the empirical risk over a finite dataset, which is an approximation of $\mathcal R$.

While we express the loss using a full covariance matrix $\pcmh$ for the sake of generality, in this work we estimate only the main diagonal, which corresponds to the the per-pixel posterior variances. 

In the next section we derive general stationarity conditions for $\hat u(v), \pcmh$ to be minimizers of the proposed self-supervised risk \eqref{eq:self-sup-risk-nll}, and show that $\mathbb E\{u|v\}$, and $\pcm$ are among the minimizers. Furthermore, we show that in the specific case of the subsampling degradation model \eqref{eq:degradation-model-grid-shifting-z} and for a network that estimates the per-pixel variances (as opposed to the full covariance matrix), the \emph{only} minimizers are $\mathbb E\{u|v\}$ and $\text{diag}(\pcm)$\footnote{
For a square $n\times n$ matrix $B$ and a vector $b\in\mathbb R^n$, $\text{diag}(B)\in\mathbb R^n$ denotes the diagonal of matrix $B$, and $\text{diag}(b)$ denotes a $n\times n$ diagonal matrix with $b$ on its diagonal.} implying that a network trained with this loss will correctly approximate the posterior mean and per-pixel variances.

These results require a precise estimation of $\ncmv$.

\paragraph{Estimation of $\ncmvh$.} 
We model the noise $n$ as signal dependent Gaussian noise with zero mean and variance dependent on the clean value of the pixel, i.e. $z_i = (A_\tau u)_i + n_i$ with $n_i = \sqrt{g((A_\tau u)_i)} r_i$, for $r_i \sim \mathcal N(0,1)$. The function $g:\mathbb R\to\mathbb R$ determines the variance of the noise from the clean signal. In this setting $\ncmv = \text{diag}(\mathbb E\{g(A_\tau u)|v\})$. 

In practice an affine variance function $g(s) = as + b$ is commonly used to model the shot and readout noise, which results in $\ncmv = \text{diag}(aA_\tau\mathbb E\{ u|v\} + b)$. Thus the network output $\hat u(v)$ can be used to estimate $\ncmvh$ as $\ncmvh = \text{diag}(aA_\tau\hat u(v) + b)$.

\section{Optimal estimators}
\label{sec:optimal_estimator}

To minimize the Bayesian risk $\mathcal R$ in \eqref{eq:self-sup-risk-nll} we minimize the inner expectation as a function of $v$
\begin{equation}\label{eq:risk_full_conditional}
\mathcal{R}_v\bigl[\hat u, \hat \Sigma \bigr]
\;=\,
\mathbb{E}_{\tau,z\mid v}\Bigl[
\mathcal{L}_{\mathrm{NLL}}(\hat{u}(v), \hat \Sigma(v), z)
\Bigr].
\end{equation}
We denote this ``per-input risk'' as $\mathcal R_v$.

In this section we establish the relation between the optimal estimators that minimize the self-supervised risk, $\hat u(v), \pcmh$, and the mean and covariance matrix of the posterior distribution $\mathbb E[u|v]$ and $\pcm = \mathbb E[(u - \mathbb E[u|v])(u - \mathbb E[u|v])^T|v]$. 

\begin{lemma}[Full covariance and general linear degradation.]\label{lemma:optimal_condition_full}
    Assume that $\tau$ and $u$ are conditionally independent given v and $z = A_\tau u + n$, where $n$ is zero mean additive noise with signal dependent variance.
    Denote 
    \begin{align}
    \acah &= (A_\tau\pcmh A_\tau^T + \ncmvh)^{-1}, \text{ and }\\
    \aca &= (A_\tau\mathbb E_{u|v}[(u - \hat u(v))(u - \hat u(v))^T] A_\tau^T + \ncmv)^{-1}.\nonumber
    \end{align}
    Then
    estimators $\hat u(v)$ and $\pcmh$ that minimize the self-supervised risk~\eqref{eq:risk_full_conditional} verify the following conditions:
    \begin{align}
        &\E_{\tau\mid v}\Bigl[A_\tau^T \acah A_\tau \Bigr] \hat u (v) = \E_{\tau\mid v} \Bigl[A_\tau^T  \acah A_\tau \Bigr] \E_{u\mid v}[u] \label{eq:stationarity-u}\\
        & \E_{\tau \mid v}\Big[A_\tau^T (\acah - \acah \aca^{-1} \acah) A_\tau\Big] = 0. \label{eq:stationarity-cov}
    \end{align}
\end{lemma}

Note that $\hat u(v) = \mathbb E\{u|v\}$, $\pcmh = \pcm$ are a solution for the above stationarity condition. If the matrix \(\mathbb{E}_{\tau\mid v}[A_\tau^T(A_\tau\pcmh A_\tau^T + \ncmvh)^{-1}A_\tau]\) is invertible, then necessarily $\hat u (v) = \E_{u\mid v}[u]$. This is the case for  $A_\tau = S_\tau$, the grid shifting operator in \eqref{eq:degradation-model-grid-shifting-z}.

\begin{proof}

Both equations follow from zeroing the derivatives of the self-supervised risk with respect to $\hat u(v)$ and $\pcmh$.

\noindent \textit{Derivation of \eqref{eq:stationarity-u}.}
The derivative of \eqref{eq:risk_full_conditional} w.r.t. $\hat u (v)$ reads:
\[
\frac{\partial \mathcal{R}}{\partial \hat u (v)}
=\mathbb{E}_{\tau,z\mid v}\Bigl[
A_\tau^T (A_\tau \pcmh A_\tau^T \smallplus \ncmvh)^{-1}
\bigl(A_\tau\hat u (v) \smallminus z\bigr)
\Bigr],
\]
where we can identify $\acah$ as the inverted matrix in the RHS.
By setting the derivative to $0$, we have the stationarity condition for solving $\hat{u} (v)$:
\begin{align*}
 \E_{\tau,z\mid v}&\Bigl[ A_\tau^T \acah A_\tau\hat u (v) \Bigr] = \E_{\tau,z\mid v} \Bigl[
A_\tau^T \acah
z \Bigr] \\
 & = \E_{\tau,u,n\mid v} \Bigl[
A_\tau^T \acah
(A_\tau u + n) \Bigr] \\
&= \E_{\tau,u\mid v} \Bigl[
A_\tau^T \acah
A_\tau u \Bigr] + \E_{\tau,u\mid v} \E_{n|\tau,u,v}\Bigl[
A_\tau^T \acah n \Bigr] \\
&= \E_{\tau\mid v} \Bigl[
A_\tau^T \acah
A_\tau \Bigr] \E_{u\mid v}[u].
\end{align*}
In the last step, we used the assumption that $\tau$ and $u$ are conditionally independent given $v$ 
and that $n$ has zero mean. 

\noindent \textit{Derivation of \eqref{eq:stationarity-cov}.}
The derivative of~\eqref{eq:risk_full_conditional} w.r.t. $\pcmh$ reads as
\begin{multline*}    
    \frac{\partial \mathcal{R}}{\partial \pcmh} = \E_{\tau,z \mid v}\left[ \frac{\partial \mathcal{L}}{\partial \pcmh}\right] 
= \\ \E_{\tau,z \mid v}\Bigl[ A_\tau^T \acah A_\tau - A_\tau^T \acah (z - A_\tau \hat u(v))(z - A_\tau \hat u(v))^T \acah A_\tau \Bigr]
\end{multline*}
(see the Appendix \ref{app:derivative-wrt-pcmv}). The condition \eqref{eq:stationarity-cov} follows from setting the derivative to zero, 
and rewriting the 2nd term as 
\begin{align*}
&\E_{\tau,z \mid v}[A_\tau^T \acah  (z - A_\tau \hat u(v))(z - A_\tau \hat u(v))^T \acah A_\tau] = \\
&\E_{\tau \mid v}\Big[A_\tau^T \acah \big( A_\tau\E_{u\mid v}[ (u - \hat u(v)) (u -\hat u(v))^T]A_\tau^T \\
&\quad\quad\quad\quad\quad + \E_{n,u\mid \tau,v}[ n n^T] \big)\acah A_\tau \Big] 
= \\
 &\E_{\tau \mid v}\Big[A_\tau^T \acah   (A_\tau\pcm A_\tau^T + \ncmv) \acah A_\tau   \Big],
\end{align*}
where we have used the previous definition of \(\ncmv\).
Recognizing $M$ in the last step yields the RHS of \eqref{eq:stationarity-cov}.
\end{proof}

In the above formulation, the loss was defined using a full covariance matrix $\pcmh$, which allows modeling correlations between different components of $u$. While this provides a more flexible uncertainty representation \citep{dorta2018structured}, it is often computationally expensive, especially for high-dimensional outputs such as pixel-level predictions.

In practice, a common simplification is to ignore the correlation between output pixels and estimate only the per pixel variances, reducing the covariance matrix to a diagonal form. In this setting, we have $\pcmh =\text{diag}(\pvvh)$ where $\pvvh\in\mathbb R^d$ contains the estimated posterior variance for each pixel.
The Gaussian NLL loss is now considered as a function of $\hat u(v)$ and $\pvvh$. We next derive the optimal estimators for this diagonal variant.
Note that \emph{we do not need to assume} that the true posterior has a diagonal covariance matrix $\Sigma(v)$. Our goal is to estimate the diagonal of the true posterior covariance. 

\begin{lemma}[Diagonal covariance a nd general linear degradation]\label{lemma:optimal_condition_diagonal}
    We make the same assumptions as in Lemma \ref{lemma:optimal_condition_full}. In addition we assume that $\pcmh = \text{diag}(\pvvh)$, 
    and denote $\acah = (A_\tau \text{diag}(\pvvh) A_\tau^T + \ncmvh)^{-1}$. The system of stationarity conditions for the optimal $\hat u(v), \pvvh$
    is given by \eqref{eq:stationarity-u} 
    and by the following equation:
    \begin{align}
        & \E_{\tau \mid v}\Big[\text{diag}(A_\tau^T (\acah  -  \acah \aca^{-1} \acah)A_\tau)\Big]  = 0.\label{eq:stationarity-cov-diagonal}
    \end{align}
\end{lemma}
\begin{proof}
    The parameterization $\pcmh = \text{diag}(\pvvh)$, does not change the derivation of Eq. \eqref{eq:stationarity-u} in Lemma~\ref{lemma:optimal_condition_full}.
    Equation \eqref{eq:stationarity-cov-diagonal} results from zeroing the derivation of the loss w.r.t. $\pvvh$ (see Appendix \ref{sec:appendix_diagonal_para}). 
\end{proof}

Lemmas \ref{lemma:optimal_condition_full}
and \ref{lemma:optimal_condition_diagonal}
provide necessary optimality conditions for the estimators $\hat u(v)$ and $\pcmh$ or $\pvvh$. Even if the sought-after posterior mean and covariances are among the set of solutions, it is difficult to exclude other unwanted solutions,
given that the variance estimators appear as part of the matrix $\hat M$.
However for the specific degradation used in \eqref{eq:degradation-model-grid-shifting-z} $A_\tau = DS_\tau$, one can show that the unique solutions are the true posterior mean and variances.

\paragraph{Subsampling matrix \( A_\tau \).}
In our self-supervised super-resolution setting, we consider high-resolution images of size $2H\times 2W$, and a specific $2\times$ subsampling matrix $D$ such that $(Du)_\bl = u_{2\bl}$, where $\bl \in \llbracket 0,H-1\rrbracket\times \llbracket0,W-1\rrbracket$. Let $S_\tau$ ($\tau\sim \mathcal U\{0,1\}^2$) be a random shift operator that shifts the high resolution image by one pixel in a random direction prior to subsampling. We then denote the combined operation as $(A_\tau u)_\bl = (DS_\tau u)_\bl = u_{2\bl + \tau}$.
$A_\tau$ is a $HW\times 4HW$ matrix where each row has all elements zero except for a single element which is one. 
\begin{proposition}[Diagonal covariance and subsampling degradation]
    \label{prop:optimal-estimators-for-random-subsampling-matrix}
    Let $A_\tau$ be the random subsampling matrix given as defined above. Assume that the observation is corrupted by additive signal-dependent noise, and that the estimated noise covariances are exact at the diagonal, i.e. $\text{diag}(\ncmvh) = \text{diag}(\ncmv)$, for all $v,\tau$. Then, solving the stationarity conditions~\eqref{eq:stationarity-u} and \eqref{eq:stationarity-cov-diagonal} leads to the following optimal estimators:
    \begin{equation}\label{eq:diag_optimal_estimate}
        \hat u(v) = \mathbb E[u|v],\quad\quad
        \pvvh = \text{diag}(\Sigma(v)).
    \end{equation}
\end{proposition}
\begin{proof}
The expectation with respect to $\tau$ in both \eqref{eq:stationarity-u} and \eqref{eq:stationarity-cov-diagonal} corresponds to a sum over the four possible shifts, yielding 
\begin{align}
    \frac{1}{4}\sum_{\tau\in\{0,1\}^2}\left[A_\tau^T \acah A_\tau\right] \left(\hat u(v) - \mathbb E[u|v]\right) = 0 \label{eq:station-eq-sum-u} \\
\frac14 \sum_{\tau\in\{0,1\}^2} \text{diag}(A_\tau^T(\acah - \acah \aca^{-1}\acah)A_\tau) = 0.
\label{eq:expectation-as-sum}
\end{align}

We begin with equation~\eqref{eq:expectation-as-sum}.
For \(\tau \in \{0,1\}^2\), let $\mathcal{J}_\tau$ denote the set of indices of all columns of  $A_\tau$ that contain at least one nonzero element.  Because each pixel index $k$ is sampled by (only) one of the $A_\tau$, $\mathcal{J}_\tau$s are disjoint and their union covers the entire index set (i.e. the entire HR domain).   For any sampled pixel $k \in \mathcal{J}_\tau$, let $l = l(k)$ denote the corresponding index in the subsampled space. Note that $k$ and $l$ are indices, and their corresponding 2D coordinates, denoted as $\bk$ and $\bl$, satisfy $\bk = 2\bl + \tau$.

For a $HW\times HW$ matrix $Q$, the projection \(A_\tau^T Q A_\tau\)  
``upsamples'' the matrix by inserting zeros for the rows and columns outside $\mathcal J_\tau$.
This preserves the diagonal correspondence such that 
\[ 
    [A_\tau^T Q A_\tau]_{kk} = \begin{cases} Q_{ll} & \text{if } k \in \mathcal{J}_\tau \\ 0 & \text{otherwise.} \end{cases}
\]
Consequently, each diagonal element of the larger matrix (high resolution) is determined by only one $A_\tau$. 

Based on our assumptions on the noise $\acah = (A_\tau \text{diag}(\pvvh) A_\tau^T + \hat R(v))^{-1}$ is a diagonal matrix with entries \( \acah_{ll} = 1/(\pvvh_k + \ncmvh_{ll}) \).
Therefore, the sum $\sum_{\tau\in\{0,-1\}^2}\left[A_\tau^T \acah A_\tau\right]$ gives a diagonal matrix with non-zero elements, and thus \eqref{eq:station-eq-sum-u} holds iff $\hat u(v) = \mathbb E[u|v]$. Note that for $\hat u(v) = \mathbb E[u|v]$, $\aca = (A_\tau\pcm A_\tau^T + \ncmv)^{-1}$.

For the same reason, equation \eqref{eq:expectation-as-sum} is verified iff
the diagonal elements of the inner terms are zero, i.e. $\text{diag}(\acah - (\acah \aca^{-1} \acah)) = 0 $ for all $\tau$. 

Furthermore, although \(\aca^{-1} \) is dense, the  diagonal of \(\acah \aca^{-1} \acah\) does not involve the off-diagonal terms of \(\aca^{-1}\) because of the diagonal nature of \(\acah\). Therefore,
\begin{equation}\label{eq:diag_ele_part2}    
    (\acah \aca^{-1} \acah)_{ll} =  \frac{\pcm_{kk} + \ncmv_{ll}}{(\pvvh_k + \ncmvh_{ll})^2},
\end{equation}
where \(\pcm_{kk}\) and \(\ncmv_{ll}\) represent the posterior variance at the sampled pixel indexed $k$ and noise level at the subsampled position $l$. 
Thus, solving for $\pvvh_k$ we have
\begin{equation}
    \pvvh_k = \pcm_{kk} + \ncmv_{ll} - \ncmvh_{ll} = \pcm_{kk}, \forall k\in \mathcal{J}_\tau.
\end{equation}
Considering all $\tau \in \{0,1\}^2$ completely defines $\pvvh$.
\end{proof}

\begin{remark}
    Proposition 
    \ref{prop:optimal-estimators-for-random-subsampling-matrix}
    shows that the optimal estimator $\pvvh$ is the aleatoric uncertainty $\text{diag}(\Sigma(v))$. This relies on the accurate estimation of $\E[u|v]$ by $\hat u(v)$. In practice, due to the limited training and network capacity, $\hat u(v)$ will have errors. In this case, the optimal $\pvvh$ consists of the aleatoric uncertainty and the squared bias, i.e.
    \[\pvvh_k = \Sigma(v)_{kk} + (\hat u(v)_k - \E[u_k|v])^2,\]
    thus it gives the expected error or point-wise risk~\citep{lahlou2023deup}, 
    which is informative for the downstream tasks. 
\end{remark}

The previous result requires an accurate estimation the variance $\text{diag}(\ncmv)$ of the observation noise. During training $\text{diag}(\ncmv)$ is used to correct the per-pixel variances $\pvvh$ for computing the NLL loss. Without this correction, the network would learn to output uncorrected variances \(\text{diag}\left(\pcm + \sum_\tau A_\tau^T\ncmv A_\tau\right),\)
overestimating the posterior variance by the noise variance.
The correction could be done during inference by subtracting $\sum_\tau A_\tau^T\ncmv A_\tau$ 
but this could result in negative variances.

\begin{table}[t!]
  \centering
  \caption{Evaluation of the estimates for both reconstruction and uncertainty estimate.}
  \begin{tabular}{@{}lcc@{}}
    \toprule
    Methods  & self-sup.  & supervised  \\
    \midrule
    $\hat \mu(v)$ PSNR & 54.05    & 54.24  \\
    $\hat \nu(v)$ V-RMSE ($\times10^{-5}$) & 2.50      & 2.51      \\
    $\hat \nu(v)$ sharpness ($\times10^{-3}$) &  5.29         &  5.81        \\
    \bottomrule
  \end{tabular}
  \label{tab:evaluation}
\end{table}

\begin{table}[t!]
  \centering 
  \caption{Comparison of uncertainty estimates for pixels from different positions. For each metric we show four values corresponding to the pixels in the four sets: top left, top right, bottom left, bottom right. The estimates for the top-left pixel has the smallest RMSE and the smallest estimated $\hat \nu$ in all four methods. Pixels from this position are used to generate LR samples.
  }
  \begin{tabular}{l|cc|cc} 
    \toprule
    Metrics & \multicolumn{2}{c |}{self-sup.} & \multicolumn{2}{c}{supervised} \\
    \toprule
    
    \multirow{2}{*}{V-RMSE ($\times10^{-5}$)}    & 0.79  & 1.90  & 0.86  & 1.87  \\
                                & 1.80  & 1.75  & 2.04  & 1.71 \\\midrule
    \multirow{2}{*}{ECE}                        & 0.060 & 0.041 & 0.021 & 0.023 \\
                                                & 0.038 & 0.032 & 0.024 & 0.027 \\\midrule
    \multirow{2}{*}{Sharpness ($\times10^{-3}$)} & 4.22  & 5.49  & 5.04  & 6.00  \\
                                                & 5.81  & 5.63  & 6.31  & 5.89 \\\midrule
    \multirow{2}{*}{RMSE ($\times10^{-3}$)}     & 1.26  & 1.53  & 1.21  & 1.47  \\
                                & 1.60  & 1.46  & 1.54  & 1.39 \\
    \bottomrule
  \end{tabular}
  \label{tab:uncer-metric-by-position}
\end{table}

\section{Experiments}
\label{sec:results}

For the experiments, we adopt the model architecture from~\citet{nguyen2022self}, as described in Section \ref{sec:method}. We modify the decoder to output an additional channel for uncertainty estimates. To demonstrate that the self-supervised method produces meaningful uncertainty estimates, we train it on a synthetic dataset (where we have the HR ground truth), and compare its results with those from the same model but trained with the Gaussian NLL loss \eqref{eq:sup-gaussian-nll} with supervision from the HR ground truth.

We use the same synthetic dataset generated from L1B products as in~\citet{nguyen2022self}. 
This simulated dataset treats SkySat L1B products as ground-truth HR images. From each HR patch, a sequence of low-resolution observations is generated by first applying random sub-pixel translations and then performing a $\times 2$ subsampling operation (without blur) to mimic the sampling process of the SkySat push-frame sensor. Each frame in the sequence is also assigned a simulated exposure time $e_i = \gamma^{c_i}$, where $c_i$ is uniformly sampled from $\{-5,\dots,5\}$, and $\gamma \sim\mathcal{U}(1.2, 1.4)$ simulating a realistic range of brightness variations.
During the training, additive white Gaussian noise with standard dev. between 5 and 18 is added to the LR observations. This yields a sequence of shifted noisy LR observations paired with the original L1B HR image as the supervision target. 
An LR multi-exposure burst is shown in Figure \ref{fig:comparison_reconstruction}.

\subsection{Quantitative evaluation}

As discussed in Section~\ref{sec:optimal_estimator}, the optimal estimators for applying the Gaussian NLL loss are $\hat{u}(v)_k = \mathbb{E}_{u\mid v}[u_k]$ for restorations and $\pvvh_k =\mathbb{E}_{u\mid v}[(u_k - \hat{u}(v)_k)^2] $ for uncertainty estimates. We use the PSNR metric to evaluate the restoration quality of the proposed method. Since the uncertainty estimates are the expected squared difference between the ground truth and the restoration, we  compare the uncertainty estimates with the actual squared prediction errors,
via the following \emph{variance RMSE}:
\[ 
\text{V-RMSE} = \sqrt{  \frac{1}{4HW} \sum^{4HW}_{k=1} \left(\hat \nu_k - (\hat u(v)_k - u_k)^2\right)^2}. \]
V-RMSE serves as a metric for the quality of uncertainty estimates. 
Table~\ref{tab:evaluation} shows the quantitative results for both reconstruction and uncertainty estimates.

\paragraph{Coverage under a Gaussian probabilistic model.}
A common way to evaluate uncertainty is via a coverage test, which defines confidence intervals for a series of nominal confidence levels $\alpha$, and compares the empirical coverage rates with the predicted coverage $\alpha$.
The definition of the confidence interval relies on assuming a Gaussian probabilistic model for the posterior distribution, i.e. 
$ u |v \sim \mathcal{N} (\hat \mu(v), \Sigma(v)).$ 
The empirical coverage is computed as the proportion of pixels in the test set for which the ground-truth value falls inside the predicted interval. We plot the coverage test results in Figure~\ref{fig:comparison_reconstruction} (right). We also report the sharpness of the intervals, which is the average length of the $90\%$ confidence intervals, see Table~\ref{tab:evaluation}. 
We can see in Table~\ref{tab:evaluation} that the supervised model achieves higher PSNR (0.2dB) on the reconstructed HR image. This is consistent with the results obtained in \citep{nguyen2022self}. 
In terms of variance RMSE both models are comparable. The sharpness and the coverage plot seem to suggest that the self-supervised variance slightly underestimates the true posterior variance.

\subsection{Qualitative evaluation}

The uncertainty estimates map in Figure~\ref{fig:comparison_reconstruction} (center) shows a checkerboard effect with a period of two, which inspires to investigate the estimates at 4 different subgrids (corresponding to the four downsampling patterns indicated by the colors in the super-resolved image of Figure \ref{fig:data_flow}). 

We compare the quality of the uncertainty estimates for these four subgrids. We show the resulting metrics in Table~\ref{tab:uncer-metric-by-position}. 
To summarize the calibration curve, we report the Calibration Error (CE)~\citep{kuleshov2018accurate,chung2021uncertainty} which is computed as 
\(\text{CE} = \frac{1}{m}\sum_{j=1}^m |p_j - \hat p_j|\), for each pair of nominal and actual confidence level \(p_j, \hat p_j\).
Coverage plots for the 4 subsets of pixels (see supplementary) show no signs of deviations from  the average coverage curve.
We see that pixels from the top-left have the smallest RMSE, meaning they match the actual squared prediction error the best. In addition, the estimated $\pvvh$ (predicted squared error) is also the smallest among the four positions.

The reason for this comes from the fact that during inference, the reference frame is not removed from the input (as it was done during training). This additional datum provides valuable information about the even pixels of the HR frame, as it is a direct (noisy) observation of those pixels, which creates a reduction in the MSE for those pixels.
Interestingly, the variance computed by the network captured this effect even with the self-supervised training.

\section{Conclusion}\label{sec:conclusion}

We proposed a self-supervised version of the Gaussian NLL loss for estimation of the aleatoric uncertainty in image restoration problems where the available observations are degraded with a linear degradation operator and contaminated by noise. 
We study the minimizers of the associated risk, and derive optimality conditions for general linear degradations, and for the specific case of super-resolution with a randomized subsampling operator. 
We justify the validity of the proposed loss by showing that
the posterior mean and variance are among the optimal estimators (in the general linear degradation case), and are the unique minimizers for the randomized subsampling operator, thus showing the equivalence with the supervised Gaussian NLL loss.
Finally, we apply this loss to the satellite image super-resolution task on a synthetic satellite image dataset simulated based on L1B~\citep{nguyen2022self}, demonstrating obtaining a performance comparable with supervised training.

Future work should explore the application of the proposed loss to other inverse problems. The challenge here is that in the general linear case the posterior mean and variance are not the unique minimizers.

{\small
\noindent\textbf{Acknowledgments. }
Project PID2024-162897NA-I00 funded by MICIU/AEI/10.13039/501100011033/ FEDER, UE.
This work was performed using HPC resources from GENCI-IDRIS
(grants 2023-AD011011801R3, 2023-AD011012453R2,
2023-AD011012458R2) and from the “Mésocentre” computing center of CentraleSupélec and ENS Paris-Saclay
supported by CNRS and Région Île-de-France (http://mesocentre.centralesupelec.fr/).
Centre
Borelli is also with Université Paris Cité, SSA and INSERM.}

\bibliography{refs}

\begin{thebibliography}{52}
\providecommand{\natexlab}[1]{#1}
\providecommand{\url}[1]{\texttt{#1}}
\expandafter\ifx\csname urlstyle\endcsname\relax
  \providecommand{\doi}[1]{doi: #1}\else
  \providecommand{\doi}{doi: \begingroup \urlstyle{rm}\Url}\fi

\bibitem[Aggarwal et~al.(2022)Aggarwal, Pramanik, John, and
  Jacob]{aggarwal2022ensure}
Hemant~Kumar Aggarwal, Aniket Pramanik, Maneesh John, and Mathews Jacob.
\newblock Ensure: A general approach for unsupervised training of deep image
  reconstruction algorithms.
\newblock \emph{IEEE transactions on medical imaging}, 42\penalty0
  (4):\penalty0 1133--1144, 2022.

\bibitem[Aiyetigbo et~al.(2024)Aiyetigbo, Korte, Anderson, Chalhoub, Kalivas,
  Luo, and Li]{aiyetigbo2024unsupervised}
Mary Aiyetigbo, Alexander Korte, Ethan Anderson, Reda Chalhoub, Peter Kalivas,
  Feng Luo, and Nianyi Li.
\newblock Unsupervised microscopy video denoising.
\newblock In \emph{Proceedings of the IEEE/CVF Conference on Computer Vision
  and Pattern Recognition}, pages 6874--6883, 2024.

\bibitem[Altekr{\"u}ger and Hertrich(2023)]{altekruger2023wppnets}
Fabian Altekr{\"u}ger and Johannes Hertrich.
\newblock Wppnets and wppflows: The power of wasserstein patch priors for
  superresolution.
\newblock \emph{SIAM Journal on Imaging Sciences}, 16\penalty0 (3):\penalty0
  1033--1067, 2023.

\bibitem[Angelopoulos et~al.(2022)Angelopoulos, Kohli, Bates, Jordan, Malik,
  Alshaabi, Upadhyayula, and Romano]{angelopoulos2022image}
Anastasios~N Angelopoulos, Amit~Pal Kohli, Stephen Bates, Michael Jordan,
  Jitendra Malik, Thayer Alshaabi, Srigokul Upadhyayula, and Yaniv Romano.
\newblock Image-to-image regression with distribution-free uncertainty
  quantification and applications in imaging.
\newblock In \emph{International Conference on Machine Learning}, pages
  717--730. PMLR, 2022.

\bibitem[Batson and Royer(2019)]{batson2019noise2self}
Joshua Batson and Loic Royer.
\newblock Noise2self: Blind denoising by self-supervision.
\newblock In \emph{International conference on machine learning}, pages
  524--533. PMLR, 2019.

\bibitem[Bhat et~al.(2023)Bhat, Gharbi, Chen, Van~Gool, and Xia]{bhat2023self}
Goutam Bhat, Micha{\"e}l Gharbi, Jiawen Chen, Luc Van~Gool, and Zhihao Xia.
\newblock Self-supervised burst super-resolution.
\newblock In \emph{Proceedings of the IEEE/CVF international conference on
  computer vision}, pages 10605--10614, 2023.

\bibitem[Bickford~Smith et~al.(2025)Bickford~Smith, Kossen, Trollope, Van
  Der~Wilk, Foster, and Rainforth]{bickford2025rethinking}
Freddie Bickford~Smith, Jannik Kossen, Eleanor Trollope, Mark Van Der~Wilk,
  Adam Foster, and Tom Rainforth.
\newblock Rethinking aleatoric and epistemic uncertainty.
\newblock In \emph{Proceedings of the 42nd International Conference on Machine
  Learning}, volume 267 of \emph{Proceedings of Machine Learning Research},
  pages 4345--4359. PMLR, 13--19 Jul 2025.
\newblock URL \url{https://proceedings.mlr.press/v267/bickford-smith25a.html}.

\bibitem[Bishop(1994)]{bishop1994mixture}
Christopher~M Bishop.
\newblock Mixture density networks.
\newblock 1994.

\bibitem[Blundell et~al.(2015)Blundell, Cornebise, Kavukcuoglu, and
  Wierstra]{blundell2015weight}
Charles Blundell, Julien Cornebise, Koray Kavukcuoglu, and Daan Wierstra.
\newblock Weight uncertainty in neural network.
\newblock In \emph{International conference on machine learning}, pages
  1613--1622. PMLR, 2015.

\bibitem[Braverman et~al.(2021)Braverman, Hobbs, Teixeira, and
  Gunson]{braverman2021post}
Amy Braverman, Jonathan Hobbs, Joaquim Teixeira, and Michael Gunson.
\newblock Post hoc uncertainty quantification for remote sensing observing
  systems.
\newblock \emph{SIAM/ASA Journal on Uncertainty Quantification}, 9\penalty0
  (3):\penalty0 1064--1093, 2021.

\bibitem[Chen et~al.(2021)Chen, Tachella, and Davies]{chen2021equivariant}
Dongdong Chen, Juli\'an Tachella, and Mike~E Davies.
\newblock Equivariant imaging: Learning beyond the range space.
\newblock In \emph{IEEE/CVF International Conference on Computer Vision
  (ICCV)}, 2021.

\bibitem[Chung et~al.(2023)Chung, Kim, Mccann, Klasky, and
  Ye]{chung2023diffusion}
Hyungjin Chung, Jeongsol Kim, Michael~Thompson Mccann, Marc~Louis Klasky, and
  Jong~Chul Ye.
\newblock Diffusion posterior sampling for general noisy inverse problems.
\newblock In \emph{The Eleventh International Conference on Learning
  Representations}, 2023.
\newblock URL \url{https://openreview.net/forum?id=OnD9zGAGT0k}.

\bibitem[Chung et~al.(2021)Chung, Char, Guo, Schneider, and
  Neiswanger]{chung2021uncertainty}
Youngseog Chung, Ian Char, Han Guo, Jeff Schneider, and Willie Neiswanger.
\newblock Uncertainty toolbox: an open-source library for assessing,
  visualizing, and improving uncertainty quantification.
\newblock \emph{arXiv preprint arXiv:2109.10254}, 2021.

\bibitem[Daras et~al.(2023)Daras, Shah, Dagan, Gollakota, Dimakis, and
  Klivans]{daras2023ambient}
Giannis Daras, Kulin Shah, Yuval Dagan, Aravind Gollakota, Alex Dimakis, and
  Adam Klivans.
\newblock Ambient diffusion: Learning clean distributions from corrupted data.
\newblock \emph{Advances in Neural Information Processing Systems},
  36:\penalty0 288--313, 2023.

\bibitem[Dewil et~al.(2021)Dewil, Anger, Davy, Ehret, Facciolo, and
  Arias]{dewil2021self}
Val{\'e}ry Dewil, J{\'e}r{\'e}my Anger, Axel Davy, Thibaud Ehret, Gabriele
  Facciolo, and Pablo Arias.
\newblock Self-supervised training for blind multi-frame video denoising.
\newblock In \emph{Proceedings of the IEEE/CVF winter conference on
  applications of computer vision}, pages 2724--2734, 2021.

\bibitem[Dorta et~al.(2018)Dorta, Vicente, Agapito, Campbell, and
  Simpson]{dorta2018structured}
Garoe Dorta, Sara Vicente, Lourdes Agapito, Neill~DF Campbell, and Ivor
  Simpson.
\newblock Structured uncertainty prediction networks.
\newblock In \emph{Proceedings of the IEEE conference on computer vision and
  pattern recognition}, pages 5477--5485, 2018.

\bibitem[Ebel et~al.(2023)Ebel, Garnot, Schmitt, Wegner, and
  Zhu]{ebel2023uncrtaints}
Patrick Ebel, Vivien Sainte~Fare Garnot, Michael Schmitt, Jan~Dirk Wegner, and
  Xiao~Xiang Zhu.
\newblock Uncrtaints: Uncertainty quantification for cloud removal in optical
  satellite time series.
\newblock In \emph{Proceedings of the IEEE/CVF Conference on Computer Vision
  and Pattern Recognition}, pages 2086--2096, 2023.

\bibitem[Ehret et~al.(2019{\natexlab{a}})Ehret, Davy, Arias, and
  Facciolo]{ehret2019joint}
Thibaud Ehret, Axel Davy, Pablo Arias, and Gabriele Facciolo.
\newblock Joint demosaicking and denoising by fine-tuning of bursts of raw
  images.
\newblock In \emph{Proceedings of the ieee/cvf international conference on
  computer vision}, pages 8868--8877, 2019{\natexlab{a}}.

\bibitem[Ehret et~al.(2019{\natexlab{b}})Ehret, Davy, Morel, Facciolo, and
  Arias]{ehret2019model}
Thibaud Ehret, Axel Davy, Jean-Michel Morel, Gabriele Facciolo, and Pablo
  Arias.
\newblock Model-blind video denoising via frame-to-frame training.
\newblock In \emph{Proceedings of the IEEE/CVF conference on computer vision
  and pattern recognition}, pages 11369--11378, 2019{\natexlab{b}}.

\bibitem[Gal and Ghahramani(2015)]{gal2015bayesian}
Yarin Gal and Zoubin Ghahramani.
\newblock Bayesian convolutional neural networks with bernoulli approximate
  variational inference.
\newblock \emph{arXiv preprint arXiv:1506.02158}, 2015.

\bibitem[Gal and Ghahramani(2016)]{gal2016dropout}
Yarin Gal and Zoubin Ghahramani.
\newblock Dropout as a bayesian approximation: Representing model uncertainty
  in deep learning.
\newblock In \emph{international conference on machine learning}, pages
  1050--1059. PMLR, 2016.

\bibitem[Gawlikowski et~al.(2023)Gawlikowski, Tassi, Ali, Lee, Humt, Feng,
  Kruspe, Triebel, Jung, Roscher, et~al.]{gawlikowski2023survey}
Jakob Gawlikowski, Cedrique Rovile~Njieutcheu Tassi, Mohsin Ali, Jongseok Lee,
  Matthias Humt, Jianxiang Feng, Anna Kruspe, Rudolph Triebel, Peter Jung,
  Ribana Roscher, et~al.
\newblock A survey of uncertainty in deep neural networks.
\newblock \emph{Artificial Intelligence Review}, 56\penalty0 (Suppl
  1):\penalty0 1513--1589, 2023.

\bibitem[Gruber et~al.(2025)Gruber, Schenk, Schierholz, Kreuter, and
  Kauermann]{gruber2025sources}
Cornelia Gruber, Patrick~Oliver Schenk, Malte Schierholz, Frauke Kreuter, and
  G{\"o}ran Kauermann.
\newblock Sources of uncertainty in supervised machine learning--a
  statisticians’ view.
\newblock \emph{arXiv preprint ArXiv:2305.16703}, 2025.

\bibitem[Haas and Rabus(2021)]{haas2021uncertainty}
Jarrod Haas and Bernhard Rabus.
\newblock Uncertainty estimation for deep learning-based segmentation of roads
  in synthetic aperture radar imagery.
\newblock \emph{Remote Sensing}, 13\penalty0 (8):\penalty0 1472, 2021.

\bibitem[Hendriksen et~al.(2020)Hendriksen, Pelt, and
  Batenburg]{hendriksen2020noise2inverse}
Allard~Adriaan Hendriksen, Dani{\"e}l~Maria Pelt, and K~Joost Batenburg.
\newblock Noise2inverse: Self-supervised deep convolutional denoising for
  tomography.
\newblock \emph{IEEE Transactions on Computational Imaging}, 6:\penalty0
  1320--1335, 2020.

\bibitem[H{\"u}llermeier and Waegeman(2021)]{hullermeier2021aleatoric}
Eyke H{\"u}llermeier and Willem Waegeman.
\newblock Aleatoric and epistemic uncertainty in machine learning: An
  introduction to concepts and methods.
\newblock \emph{Machine learning}, 110\penalty0 (3):\penalty0 457--506, 2021.

\bibitem[Kawar et~al.(2022)Kawar, Elad, Ermon, and Song]{kawar2022denoising}
Bahjat Kawar, Michael Elad, Stefano Ermon, and Jiaming Song.
\newblock Denoising diffusion restoration models.
\newblock \emph{Advances in neural information processing systems},
  35:\penalty0 23593--23606, 2022.

\bibitem[Kendall and Gal(2017)]{kendall2017uncertainties}
Alex Kendall and Yarin Gal.
\newblock What uncertainties do we need in bayesian deep learning for computer
  vision?
\newblock \emph{Advances in neural information processing systems}, 30, 2017.

\bibitem[Kim and Ye(2021)]{kim2021noise2score}
Kwanyoung Kim and Jong~Chul Ye.
\newblock Noise2score: tweedie’s approach to self-supervised image denoising
  without clean images.
\newblock \emph{Advances in Neural Information Processing Systems},
  34:\penalty0 864--874, 2021.

\bibitem[Kirchhof et~al.(2025)Kirchhof, Kasneci, and
  Kasneci]{kirchhof2025reexamining}
Michael Kirchhof, Gjergji Kasneci, and Enkelejda Kasneci.
\newblock Reexamining the aleatoric and epistemic uncertainty dichotomy.
\newblock In \emph{The Fourth Blogpost Track at ICLR 2025}, 2025.

\bibitem[Krull et~al.(2019)Krull, Buchholz, and Jug]{krull2019noise2void}
Alexander Krull, Tim-Oliver Buchholz, and Florian Jug.
\newblock Noise2void-learning denoising from single noisy images.
\newblock In \emph{Proceedings of the IEEE/CVF conference on computer vision
  and pattern recognition}, pages 2129--2137, 2019.

\bibitem[Krull et~al.(2020)Krull, Vi{\v{c}}ar, Prakash, Lalit, and
  Jug]{krull2020probabilistic}
Alexander Krull, Tom{\'a}{\v{s}} Vi{\v{c}}ar, Mangal Prakash, Manan Lalit, and
  Florian Jug.
\newblock Probabilistic noise2void: Unsupervised content-aware denoising.
\newblock \emph{Frontiers in Computer Science}, 2:\penalty0 5, 2020.

\bibitem[Kuleshov et~al.(2018)Kuleshov, Fenner, and
  Ermon]{kuleshov2018accurate}
Volodymyr Kuleshov, Nathan Fenner, and Stefano Ermon.
\newblock Accurate uncertainties for deep learning using calibrated regression.
\newblock In \emph{International conference on machine learning}, pages
  2796--2804. PMLR, 2018.

\bibitem[Lahlou et~al.(2023)Lahlou, Jain, Nekoei, Butoi, Bertin, Rector-Brooks,
  Korablyov, and Bengio]{lahlou2023deup}
Salem Lahlou, Moksh Jain, Hadi Nekoei, Victor~I Butoi, Paul Bertin, Jarrid
  Rector-Brooks, Maksym Korablyov, and Yoshua Bengio.
\newblock {DEUP}: Direct epistemic uncertainty prediction.
\newblock \emph{Transactions on Machine Learning Research}, 2023.
\newblock ISSN 2835-8856.
\newblock URL \url{https://openreview.net/forum?id=eGLdVRvvfQ}.
\newblock Expert Certification.

\bibitem[Laine et~al.(2019)Laine, Karras, Lehtinen, and Aila]{laine2019high}
Samuli Laine, Tero Karras, Jaakko Lehtinen, and Timo Aila.
\newblock High-quality self-supervised deep image denoising.
\newblock \emph{Advances in neural information processing systems}, 32, 2019.

\bibitem[Lakshminarayanan et~al.(2017)Lakshminarayanan, Pritzel, and
  Blundell]{lakshminarayanan2017simple}
Balaji Lakshminarayanan, Alexander Pritzel, and Charles Blundell.
\newblock Simple and scalable predictive uncertainty estimation using deep
  ensembles.
\newblock \emph{Advances in neural information processing systems}, 30, 2017.

\bibitem[Lehtinen et~al.(2018)Lehtinen, Munkberg, Hasselgren, Laine, Karras,
  Aittala, and Aila]{lehtinen2018noise2noise}
Jaakko Lehtinen, Jacob Munkberg, Jon Hasselgren, Samuli Laine, Tero Karras,
  Miika Aittala, and Timo Aila.
\newblock Noise2noise: Learning image restoration without clean data.
\newblock \emph{arXiv preprint arXiv:1803.04189}, 2018.

\bibitem[Metzler et~al.(2018)Metzler, Mousavi, Heckel, and
  Baraniuk]{metzler2018unsupervised}
Christopher~A Metzler, Ali Mousavi, Reinhard Heckel, and Richard~G Baraniuk.
\newblock Unsupervised learning with stein's unbiased risk estimator.
\newblock \emph{arXiv preprint arXiv:1805.10531}, 2018.

\bibitem[Millard and Chiew(2023)]{millard2023theoretical}
Charles Millard and Mark Chiew.
\newblock A theoretical framework for self-supervised mr image reconstruction
  using sub-sampling via variable density noisier2noise.
\newblock \emph{IEEE transactions on computational imaging}, 9:\penalty0
  707--720, 2023.

\bibitem[Minka(2000)]{minka2000matrix}
Thomas~P Minka.
\newblock Old and new matrix algebra useful for statistics, December 2000.
\newblock URL \url{https://tminka.github.io/papers/matrix/}.

\bibitem[Miranda et~al.(2025)Miranda, Mena, and Dengel]{miranda2025analysis}
Miro Miranda, Francisco Mena, and Andreas Dengel.
\newblock An analysis of temporal dropout in earth observation time series for
  regression tasks.
\newblock In \emph{International Symposium on Intelligent Data Analysis}, pages
  389--402. Springer, 2025.

\bibitem[Moran et~al.(2020)Moran, Schmidt, Zhong, and
  Coady]{moran2020noisier2noise}
Nick Moran, Dan Schmidt, Yu~Zhong, and Patrick Coady.
\newblock Noisier2noise: Learning to denoise from unpaired noisy data.
\newblock In \emph{Proceedings of the IEEE/CVF conference on computer vision
  and pattern recognition}, pages 12064--12072, 2020.

\bibitem[Nguyen et~al.(2021)Nguyen, Anger, Davy, Arias, and
  Facciolo]{nguyen2021self}
Ngoc~Long Nguyen, J{\'e}r{\'e}my Anger, Axel Davy, Pablo Arias, and Gabriele
  Facciolo.
\newblock Self-supervised multi-image super-resolution for push-frame satellite
  images.
\newblock In \emph{Proceedings of the IEEE/CVF Conference on Computer Vision
  and Pattern Recognition}, pages 1121--1131, 2021.

\bibitem[Nguyen et~al.(2022)Nguyen, Anger, Davy, Arias, and
  Facciolo]{nguyen2022self}
Ngoc~Long Nguyen, J{\'e}r{\'e}my Anger, Axel Davy, Pablo Arias, and Gabriele
  Facciolo.
\newblock Self-supervised super-resolution for multi-exposure push-frame
  satellites.
\newblock In \emph{Proceedings of the IEEE/CVF Conference on Computer Vision
  and Pattern Recognition}, pages 1858--1868, 2022.

\bibitem[Nix and Weigend(1994)]{nix1994estimating}
David~A Nix and Andreas~S Weigend.
\newblock Estimating the mean and variance of the target probability
  distribution.
\newblock In \emph{Proceedings of 1994 ieee international conference on neural
  networks (ICNN'94)}, volume~1, pages 55--60. IEEE, 1994.

\bibitem[Pang et~al.(2021)Pang, Zheng, Quan, and Ji]{pang2021recorrupted}
Tongyao Pang, Huan Zheng, Yuhui Quan, and Hui Ji.
\newblock Recorrupted-to-recorrupted: Unsupervised deep learning for image
  denoising.
\newblock In \emph{Proceedings of the IEEE/CVF conference on computer vision
  and pattern recognition}, pages 2043--2052, 2021.

\bibitem[Soltanayev and Chun(2018)]{soltanayev2018training}
Shakarim Soltanayev and Se~Young Chun.
\newblock Training deep learning based denoisers without ground truth data.
\newblock \emph{Advances in neural information processing systems}, 31, 2018.

\bibitem[Tachella and Davies(2026)]{tachella2026self}
Juli{\'a}n Tachella and Mike Davies.
\newblock Self-supervised learning from noisy and incomplete data.
\newblock \emph{arXiv preprint arXiv:2601.03244}, 2026.

\bibitem[Tachella and Pereyra(2024)]{tachella2024equivariant}
Juli{\'a}n Tachella and Marcelo Pereyra.
\newblock Equivariant bootstrapping for uncertainty quantification in imaging
  inverse problems.
\newblock In \emph{International Conference on Artificial Intelligence and
  Statistics}, pages 4141--4149. PMLR, 2024.

\bibitem[Tachella et~al.(2024)Tachella, Davies, and
  Jacques]{tachella2024unsure}
Juli{\'a}n Tachella, Mike Davies, and Laurent Jacques.
\newblock Unsure: self-supervised learning with unknown noise level and stein's
  unbiased risk estimate.
\newblock \emph{arXiv preprint arXiv:2409.01985}, 2024.

\bibitem[Valdenegro-Toro and Mori(2022)]{valdenegro2022deeper}
Matias Valdenegro-Toro and Daniel~Saromo Mori.
\newblock A deeper look into aleatoric and epistemic uncertainty
  disentanglement.
\newblock In \emph{2022 IEEE/CVF Conference on Computer Vision and Pattern
  Recognition Workshops (CVPRW)}, pages 1508--1516. IEEE, 2022.

\bibitem[Valsesia and Magli(2021)]{valsesia2021permutation}
Diego Valsesia and Enrico Magli.
\newblock Permutation invariance and uncertainty in multitemporal image
  super-resolution.
\newblock \emph{IEEE Transactions on Geoscience and Remote Sensing},
  60:\penalty0 1--12, 2021.

\end{thebibliography}

\clearpage

\appendix
\section{Mathematical derivations}\label{sec:appendix}

\subsection{Derivative of NLL loss with respect to $\pcmh$}
\label{app:derivative-wrt-pcmv}

Proving Lemma~\ref{lemma:optimal_condition_full} requires computing the derivative of ``per-input risk'' $\mathcal R_v$ eq.~\eqref{eq:risk_full_conditional}  w.r.t. \(\pcmh\):
\begin{equation}
    \frac{\partial \mathcal{R}_v}{\partial \pcmh} = \mathbb{E}_{\tau,z\mid v} \left[\frac{\partial \mathcal{L}_{\mathrm{NLL}}}{\partial \pcmh}\right]
\end{equation}

Recall that $\acah = (A_\tau\pcmh A_\tau^T + \ncmvh)^{-1}$, and we denote $\mathcal{L}_{\mathrm{NLL}}$ as $\mathcal{L}$ for simplicity. We can express the differential \citep{minka2000matrix} of $\mathcal L$ using the chain rule as:
\begin{align*}
& d \mathcal{L} = \mathrm{Tr} \left(\frac{\partial \mathcal{L}}{\partial \acah}^T d\acah \right) \\
&\frac{\partial \mathcal{L}}{\partial \acah}
= (z - A_\tau \hat u(v))(z - A_\tau \hat u(v))^T - \acah^{-1} \\
& d\acah = -\acah A_\tau d\hat\Sigma A_\tau^T \acah
\end{align*}
The above equations lead to 
\begin{align} 
    d \mathcal{L} 
    &= \mathrm{Tr} \left(-\frac{\partial \mathcal{L}}{\partial \acah}^T \acah A_\tau d\hat\Sigma A_\tau^T \acah\right) \notag\\
    &=  \mathrm{Tr} \left(-A_\tau^T \acah\frac{\partial \mathcal{L}}{\partial \acah}^T \acah A_\tau d\hat\Sigma \right) \label{eq:differential_Sigma}
\end{align}
Thus, 
\begin{equation}
    \frac{\partial \mathcal{L}}{\partial \hat\Sigma} = \Big(-A_\tau^T \acah\frac{\partial \mathcal{L}}{\partial \acah}^T \acah A_\tau \Bigr)^T = -A_\tau^T \acah\frac{\partial \mathcal{L}}{\partial \acah} \acah A_\tau 
\end{equation}
Expanding the term $\frac{\partial \mathcal{L}}{\partial \acah}$, we have 
\[
\frac{\partial \mathcal{L}}{\partial \hat\Sigma} = A_\tau^T \acah A_\tau - A_\tau^T \acah (z - A_\tau \hat u(v))(z - A_\tau \hat u(v))^T \acah A_\tau.
\]

\subsection{Derivative of NLL loss with respect to $\pvvh$}
\label{sec:appendix_diagonal_para}

Denote $\acah = (A_\tau \text{diag}(\pvvh) A_\tau^T)^{-1}$. 
We consider the differential expression~\eqref{eq:differential_Sigma} under the diagonal parameterization $\pcmh=\text{diag}(\pvvh)$.
The following holds for the diagonal operator:
\begin{equation*}
    d\hat\Sigma = \text{diag}(d\pvvh),
\end{equation*}
which leads to
\begin{align} 
    d \mathcal{L} 
    &= \mathrm{Tr} \left(\frac{\partial \mathcal{L}}{\partial \hat\Sigma}^T \;d\hat\Sigma\right) \notag 
    = \mathrm{Tr} \left(\frac{\partial \mathcal{L}}{\partial \hat\Sigma}^T \;\text{diag}(d\pvvh)\right) \notag \\
    &= \text{diag}\left(\frac{\partial \mathcal{L}}{\partial \hat\Sigma}\right)^T \;d\pvvh.
\end{align}
In this last equation the diag operator is applied to a matrix and thus it acts by extracting its diagional as a column vector.
We conclude that 
\begin{equation}
    \frac{\partial \mathcal{L}}{\partial \pvvh} 
    = \text{diag}\left(\frac{\partial \mathcal{L}}{\partial \hat\Sigma} \right) 
    =\text{diag}(A_\tau^T (\acah  -  \acah \aca^{-1} \acah)A_\tau)
\end{equation}

\end{document}